\documentclass[10pt]{article}
\usepackage{arxiv}

\usepackage{amsmath,amsfonts,amssymb,amsthm}

\usepackage{algorithmic}
\usepackage{algorithm}

\usepackage{graphicx}
\usepackage{array}
\usepackage{booktabs,xcolor,threeparttable}
\usepackage{tabularx}
\usepackage{makecell}
\usepackage{multirow}
\usepackage{float}     % provides the [H] float specifier for fixed figure placement

\usepackage[numbers,sort&compress]{natbib}

\usepackage{url}

\usepackage[colorlinks=true,allcolors=blue]{hyperref}

\newtheorem{theorem}{Theorem}
\newtheorem{lemma}{Lemma}
\begin{document}

\title{X-LogSMask: Expand Transformer for Graph-Structured Data}
\author{%
  Leyan Li\thanks{E-mail: li\_leyan@163.com; rangrn6907@163.com; zzx2207621676@hotmail.com; Hlp\_afeu@163.com} \And
  Rennong Yang \And
  Zhenxing Zhang \And
  Liping Hu%
}

\maketitle

\begin{abstract}
Transformers have become general-purpose architectures, but their all-to-all self-attention is poorly matched to graph data, whose interactions are sparse, structured and multi-scale. Existing Graph Transformers address this mismatch through structural encodings, hybrid message-passing modules or learned attention constraints, often introducing additional complexity and limited interpretability. Here we introduce X-LogSMask, an e\underline{x}plainable multi-head \underline{log}arithmic \underline{s}tructural \underline{mask} that injects symmetrically normalized graph topology directly into attention logits. The logarithmic transform converts structural connectivity into a topology-aware gating signal, suppressing unsupported node interactions while preserving feature-dependent attention. By assigning different powers of the normalized adjacency matrix to different attention heads, X-LogSMask gives each head a defined structural radius and supports multi-hop information propagation within a single layer. We further show that a standard Transformer encoder can be interpreted as one-step message passing on a complete graph, motivating X-LogSMask as a topology-constrained alternative to unrestricted self-attention. Across 20 node-, edge- and graph-level benchmarks, Transformers equipped with X-LogSMask achieve state-of-the-art performance on 13 datasets and remain competitive in a lightweight one-layer configuration. These results show that simple, interpretable structural masks can make self-attention an effective graph-learning operator without changing the Transformer architecture. The code is available at \url{https://github.com/LiLeyan-0120/X-LogSMask}.
\end{abstract}


\section{Introduction}

Graphs provide a natural representation for relational systems, including citation networks, social networks, molecules, biological systems and transportation networks. Unlike sequences or images, graphs are defined by irregular neighborhoods, sparse connectivity and dependencies that may span multiple topological scales. Graph neural networks (GNNs) address this structure by propagating information along observed edges, thereby aligning model computation with graph topology \cite{kipf2017semi, xu2019gin}. By contrast, the standard Transformer was developed around self-attention, which permits every token to interact with every other token \cite{vaswani2017attention}. This design gives Transformers strong representational flexibility, but it also creates a mismatch when the input is a graph.

\begin{figure}
    \centering
    \includegraphics[width=0.5\linewidth]{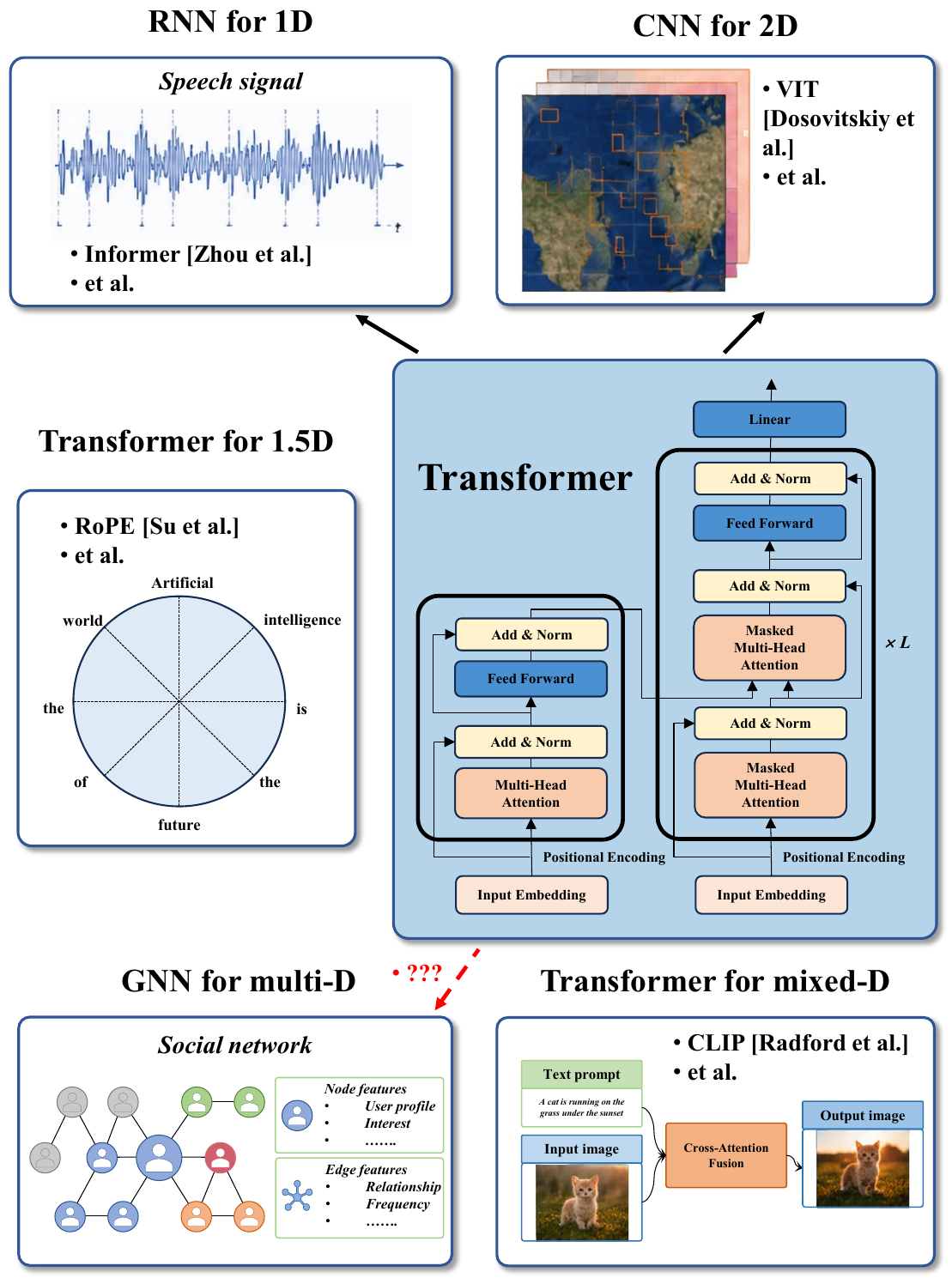}
    \caption{Transformer for all kinds of data.}
    \label{fig:nDDATA}
\end{figure}

Transformers have been successfully extended beyond language to vision, temporal modelling and other structured domains \cite{devlin2019bert, brown2020language, dosovitskiy2020image, zhou2021informer, gong2021ast}. Their success suggests that self-attention can serve as a general-purpose information aggregation operator when equipped with an appropriate inductive bias. In sequential data, this role is played by positional encodings, which supply the order information absent from permutation-invariant attention. For graph data, however, the required inductive bias is not a one-dimensional position but a sparse, relational and multi-hop topology. Without such a bias, all-to-all self-attention can mix topologically unrelated nodes and weaken the structural constraints that make message passing effective.

Recent Graph Transformer models have therefore sought to inject graph structure into attention. One line of work incorporates attention mechanisms into GNN architectures, as in Graph Attention Networks \cite{velivckovic2018graph}. Another combines message-passing modules with Transformer blocks in hybrid pipelines \cite{zhang2020transgnn, rampasek2022graphgps}. A third line directly modifies the attention mechanism through positional encodings, structural biases or learned attention constraints, including Graphormer, GradFormer and Eigenformer \cite{ying2021graphormer, liu2024gradformer, qiu2022eigenformer}. These methods have shown that structural priors are important for graph-aware self-attention. However, they often introduce additional architectural complexity, rely on implicit structural effects, or provide limited control over what each attention head captures.

Here we introduce X-LogSMask, an explainable multi-head logarithmic structural mask for adapting Transformers to graph-structured data. X-LogSMask is constructed from a symmetrically normalized adjacency matrix with self-loops and is injected additively into the attention logits. The logarithmic transform converts normalized structural connectivity into a topology-aware gating signal in attention space. As a result, unsupported node interactions are strongly suppressed, while feature-dependent attention scores remain discriminative. This design preserves the core Transformer architecture, but replaces unrestricted communication with topology-constrained message passing. X-LogSMask further assigns different powers of the normalized adjacency matrix to different attention heads. Each head is therefore associated with a defined structural radius, ranging from local neighborhoods to higher-order graph context. This head-wise decomposition gives multi-head attention an explicit structural interpretation and allows multi-hop information propagation within a single Transformer layer. It also reduces the need to stack many message-passing layers, which can aggravate over-smoothing and over-squashing in conventional GNNs.

\begin{table}[htbp]
  \centering
  \scriptsize
  \setlength{\tabcolsep}{0pt}
  \setlength{\extrarowheight}{1pt}
  \renewcommand{\arraystretch}{1.16}
  \caption{Summary of the Representative Graph Transformer Models}
  \label{tab:classic_gnn_transformer_models}
  \begin{tabularx}{\textwidth}{>{\centering\arraybackslash}p{0.15\textwidth} >{\centering\arraybackslash}p{0.06\textwidth} >{\centering\arraybackslash}p{0.27\textwidth} >{\centering\arraybackslash}X}
    \toprule
    \makecell[c]{Model} & \makecell[c]{Year} & \makecell[c]{Graph Structural Encoding} & \makecell[c]{Formulaic Representation} \\
    \midrule
    \makecell[c]{Graphormer\cite{ying2021graphormer}} & 2021 & Central, spatial, and edge encoding & $\alpha_{ij}\propto\exp\!\left(\frac{(\mathbf{h}_i\mathbf{W}^Q)(\mathbf{h}_j\mathbf{W}^K)^{\mathrm{T}}}{\sqrt{d}}+b_{\phi(v_i,v_j)}+\frac{1}{N}\sum_{n=1}^{N}\mathbf{x}_{e_n}(\mathbf{w}_n^E)^{\mathrm{T}}\right)$ \\
    \makecell[c]{Gradformer\cite{liu2024gradformer}} & 2024 & Decay mask with learnable constraints & $\alpha_{ij}\propto\exp\!\left(\frac{\mathbf{q}_i\mathbf{k}_j^{\mathrm{T}}}{\sqrt{d}}\cdot\lambda^{\mathrm{ReLU}(\psi(v_i,v_j)-sp)}\right)$ \\
    \makecell[c]{Eigenformer\cite{qiu2022eigenformer}} & 2024 & Structural attention bias & $\alpha_{ij}\propto\exp\!\left(\frac{\mathbf{q}_i\mathbf{k}_j^{\mathrm{T}}}{\sqrt{d}}+\beta\sum_{k=1}^{K}\frac{\mathbf{u}_k(i)\mathbf{u}_k(j)}{\lambda_k}\right)$ \\
    Ours & 2025 & X-LogSMask & $\alpha_{ij}\propto\exp\!\left(\frac{\mathbf{q}_i\mathbf{k}_j^{\mathrm{T}}+\log\!\left(\mathbf{D}^{-1/2}\mathbf{G}_{\mathrm{head}}\mathbf{D}^{-1/2}+\epsilon\right)}{\sqrt{d}}\right)$ \\
    \bottomrule
  \end{tabularx}
\end{table}

Our contributions are as follows:
\begin{enumerate}
  \item We propose X-LogSMask, a compact and theoretically motivated structural bias that replaces positional encodings in the vanilla Transformer. X-LogSMask is constructed from a symmetrically normalized adjacency matrix and a logarithmic transform, and is injected additively into the attention logits to suppress irrelevant communications while preserving the discriminative scaling properties of attention.
  \item An explainable multi-head mechanism is designed for LogSMask, where different heads attend to different powers of the normalized adjacency matrix. This head-wise decomposition enables each attention head to specialize in a specific structural radius, accelerating multi-hop information propagation within a single layer. Thus, our method supports an efficient 1-layer Transformer solution that is practical for large-scale graph tasks.
\end{enumerate}

The article is organized as follows. Section~II distinguishes X-LogSMask from other Graph Transformer designs. Section~III establishes the theoretical foundation that links the standard Transformer encoder to message passing. Section~IV details the architectural design of X-LogSMask. Section~V validates the model through experiments on node-, edge-, and graph-level benchmarks. Section~VI concludes the article and discusses future research directions.

\section{What’s the difference between ours and others}

The Introduction section outlined three principal paradigms for adapting Transformers to graph-structured data; our approach belongs to the third paradigm, which directly augments the self-attention mechanism with graph-derived structural biases. To clarify distinctions among methods in this family, Table \ref{tab:classic_gnn_transformer_models} summarizes representative structural-encoding strategies and their formal instantiations.

Broadly speaking, existing methods fall into two conceptual classes. Additive approaches (e.g., Graphormer\cite{ying2021graphormer} and Eigenformer\cite{qiu2022eigenformer}) incorporate structural encodings by adding bias terms to the attention logits, thereby preserving topological information in a form that is straightforwardly learnable. Multiplicative approaches (e.g., GradFormer\cite{liu2024gradformer}) instead scale the attention scores, directly modulating message-passing intensity and providing an effective mechanism to suppress spurious interactions between distant or non-adjacent nodes. Each strategy has complementary advantages: additive biases facilitate rich structural representation and gradient-based learning, whereas multiplicative biases exert stronger, direct control over inter-node communication.

X-LogSMask is additive in formulation but leverages a minimal-value logarithmic masking scheme that unifies the strengths of both paradigms. By applying a log-transformed, symmetrically normalized adjacency-derived mask to the attention logits and assigning head-specific powers of the adjacency to different heads, X-LogSMask (i) retains the representational flexibility of additive encodings, and (ii) enforces selective attenuation of irrelevant connections akin to multiplicative bias. This hybrid behavior enables explicit regulation of message-passing strength while maintaining efficient structural encoding.

Moreover, unlike prior methods that require deep stacking to capture long-range dependencies, our mask embeds multi-hop structural cues within a single layer via head-wise specialization: lower-order heads focus on local neighborhoods while higher-order heads capture progressively longer-range interactions. This design accelerates information propagation, reduces variance introduced by unconstrained multi-head attention, and yields both improved training efficiency and empirical performance.

In summary, compared to competing designs, X-LogSMask offers a low-complexity, architecturally conservative modification to the Transformer that (i) preserves the original attention mechanism, (ii) injects interpretable, multi-scale structural priors, and (iii) scales practically to a range of graph tasks. These characteristics make it a flexible building block for extending Transformer architectures to other structured modalities, including sequential and visual domains.

\section{Preliminary}
\label{sec:preliminay}

\begin{table}[htbp]
  \centering
  \scriptsize
  \setlength{\tabcolsep}{2pt}
  \renewcommand{\arraystretch}{1.10}
  \caption{Correspondence between MPNN and Transformer encoder operations.}
  \label{tab:mpnn_transformer_correspondence}
  \begin{tabularx}{\textwidth}{
    >{\centering\arraybackslash}p{0.15\textwidth}
    >{\centering\arraybackslash}X
    >{\centering\arraybackslash}p{0.16\textwidth}
    >{\centering\arraybackslash}X
  }
    \toprule
    \makecell[c]{MPNN} & \makecell[c]{Formulaic representation} & \makecell[c]{Transformer} & \makecell[c]{Implementation} \\
    \midrule
    Message generation & $m_{ij}^{(l)} = {\phi ^{(l)}}\left( {{\bf{h}}_i^{(l - 1)},{\bf{h}}_j^{(l - 1)},{e_{ij}}} \right)$ & Self-attention calculation & ${\phi ^{(l)}}( \cdot , \cdot ,1) = {\rm{softmax}}\left( {\frac{{{{\bf{q}}_i}{\bf{k}}_j^{\rm{T}}}}{{\sqrt d }}} \right) \cdot {{\bf{v}}_j}$ \\
    Aggregation & ${\bf{h}}_i^{(l),{\rm{agg}}} = { \oplus _{j \in {\cal N}(i)}}m_{ij}^{(l)}$ & Attention output & ${ \oplus _{j \in {\cal N}(i)}} = {\rm{LN}}(\sum\limits_{j \in {\cal V}} {m_{ij}^{(l)}}  + {\bf{h}}_i^{(l - 1)})$ \\
    Node update & ${\bf{h}}_i^{(l)} = {\psi ^{(l)}}({\bf{h}}_i^{(l - 1)},{\bf{h}}_i^{(l),{\rm{agg}}})$ & FFN output & ${\psi ^{(l)}} = {\rm{LN}}({\bf{h}}_i^{(l),{\rm{agg}}} + {\rm{FFN}}({\bf{h}}_i^{(l),{\rm{agg}}}))$ \\
    \bottomrule
  \end{tabularx}
\end{table}

\begin{theorem}
The standard Transformer encoder layer implements a one-step message-passing operator on the complete graph $\mathcal{G}_c=(\mathcal{V},\mathcal{E}_c)$ with $\mathcal{E}_c = \mathcal{V}\times\mathcal{V}$, and thus constitutes a special case of message passing neural networks (MPNN).
\label{the:mpnn_transformer_correspondence}
\end{theorem}

\begin{proof}\let\qed\relax
Consider the standard formulation of an MPNN layer, which is defined by a message function, a permutation-invariant aggregation operator, and an update function. At layer $l$, the message sent from node $j$ to node $i$ is given by
\begin{equation}
m_{ij}^{(l)}=\phi^{(l)}\left(\mathbf{h}_i^{(l-1)},\mathbf{h}_j^{(l-1)},e_{ij}\right)
\end{equation}
where $\mathbf{h}_i^{(l-1)}$ denotes the representation of node $i$ from the preceding layer and $e_{ij}$ denotes the edge attribute.

The aggregated representation is
\begin{equation}
\mathbf{h}_{i,\mathrm{agg}}^{(l)}
=
\bigoplus_{j\in\mathcal{N}(i)} m_{ij}^{(l)}
\end{equation}

The layer output is obtained through
\begin{equation}
\mathbf{h}_i^{(l)}
=
\psi^{(l)}\left(\mathbf{h}_i^{(l-1)},\mathbf{h}_{i,\mathrm{agg}}^{(l)}\right)
\end{equation}

For a Transformer encoder layer applied to the same node set, define $\mathbf{q}_i=\mathbf{h}_i^{(l-1)}\mathbf{W}^Q$, $\mathbf{k}_j=\mathbf{h}_j^{(l-1)}\mathbf{W}^K$, and $\mathbf{v}_j=\mathbf{h}_j^{(l-1)}\mathbf{W}^V$. The contribution of token $j$ to token $i$ can then be expressed as an attention-weighted message,
\begin{equation}
m_{ij}^{(l)}
=
\alpha_{ij}\mathbf{v}_j,\quad
\alpha_{ij}
=
\frac{
\exp\left(\mathbf{q}_i\mathbf{k}_j^{\mathrm{T}}/\sqrt{d}\right)
}{
\sum_{r\in\mathcal{V}}
\exp\left(\mathbf{q}_i\mathbf{k}_r^{\mathrm{T}}/\sqrt{d}\right)
}
\end{equation}

Accordingly, the self-attention output for token $i$ is
\begin{equation}
\mathbf{z}_i^{(l)}
=
\sum_{j\in\mathcal{V}} m_{ij}^{(l)}
\end{equation}

This expression is identical in form to MPNN aggregation when the neighborhood of each node is taken to be the full vertex set, that is, $\mathcal{N}(i)=\mathcal{V}$ for all $i$. This choice corresponds precisely to message passing on the complete graph $\mathcal{G}_c=(\mathcal{V},\mathcal{V}\times\mathcal{V})$. The subsequent residual connection, layer normalization, and feed-forward network implement an update map of the form $\psi^{(l)}$. Therefore, a standard Transformer encoder layer can be regarded as an MPNN layer instantiated on $\mathcal{G}_c$. A step-by-step correspondence is provided in Table \ref{tab:mpnn_transformer_correspondence}.
\end{proof}

\begin{lemma}
The fully-connected nature of the Transformer arises from the undiscriminating message passing among tokens induced by the self-attention mechanism.
\label{lem:mpnn_transformer_correspondence}
\end{lemma}

\begin{proof}
By construction, the softmax attention produces non-negative, normalized coefficients $\alpha_{ij}$ for every pair $(i,j)$, where the normalization for fixed $i$ sums over all $j\in\mathcal{V}$. Consequently, the aggregated message for node $i$ is

\begin{equation}
\mathbf{m}_i = \sum_{j\in\mathcal{V}} \alpha_{ij}\mathbf{v}_j
\label{eq:2}
\end{equation}
which includes contributions from every token in the input set. In the MPNN view, this means the receptive neighborhood $\mathcal{N}(i)$ equals the entire vertex set $\mathcal{V}$ for all $i$, i.e., the induced communication graph is complete. This absence of structural constraint on permitted message sources—each node both sends and receives messages to/from every other node—gives rise to the Transformer's effective full connectivity.
\end{proof}

\begin{figure}[!t]
    \centering
    \includegraphics[width=0.88\linewidth]{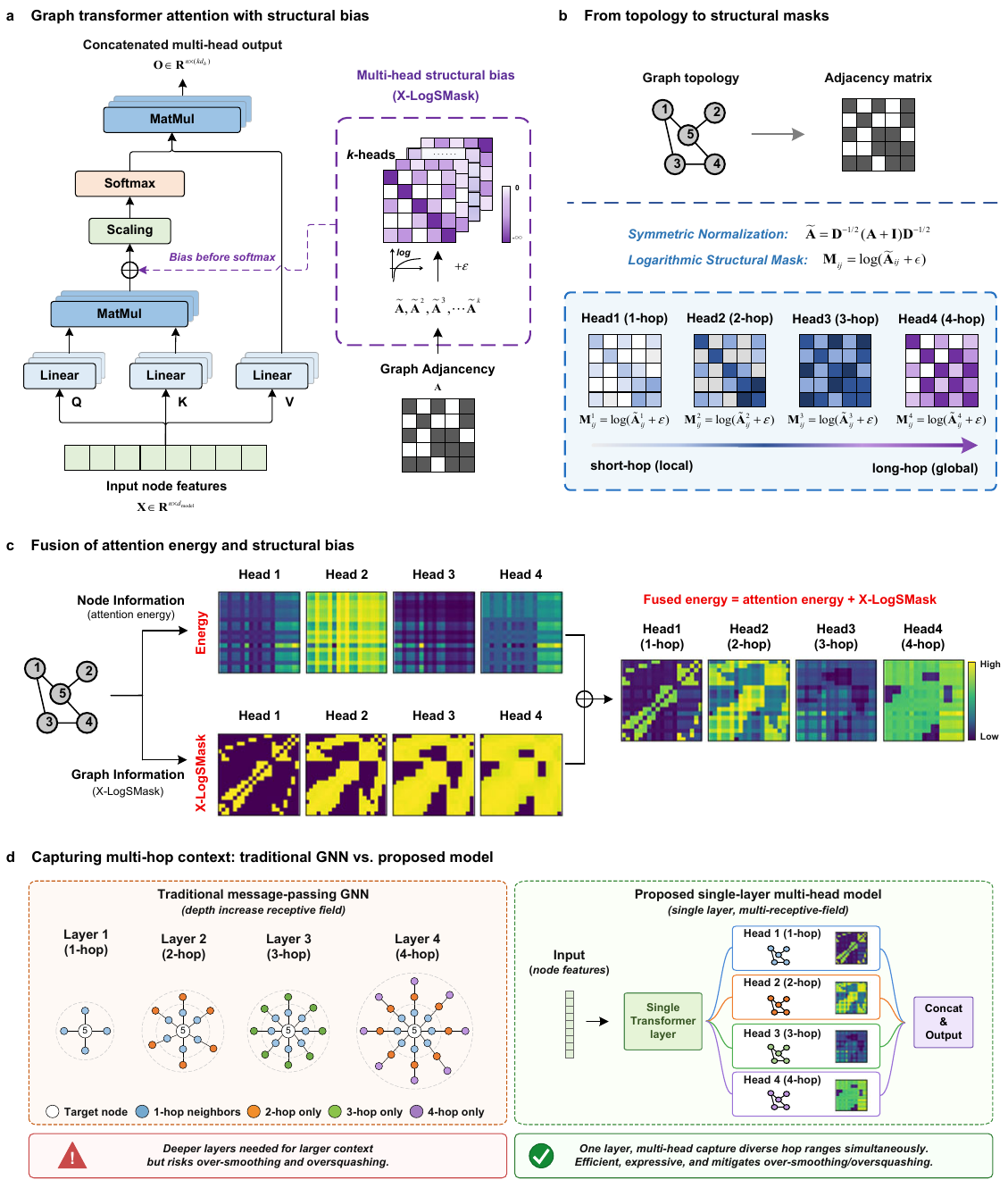}
    \caption{Model architecture and structural masking mechanism of X-LogSMask. (a) Graph Transformer attention with structural bias. Node features are linearly projected to query, key and value representations; the scaled query-key energy is combined with a multi-head X-LogSMask before softmax, thereby constraining value aggregation by graph topology. (b) Topology-to-mask construction. The graph adjacency is augmented with self-loops and symmetrically normalized to form $\hat{\mathbf{A}}$. A logarithmic transform converts normalized structural weights into attention-logit biases, and powers of $\hat{\mathbf{A}}$ yield head-specific masks from 1-hop local structure to 4-hop global context. (c) Fusion mechanism. Learned node-information energy and graph-information masks are added head by head, producing fused logits that couple feature-dependent attention with multi-hop structural priors. (d) Multi-hop context capture. Conventional message-passing GNNs expand the receptive field by stacking layers, which can aggravate over-smoothing and over-squashing. X-LogSMask assigns different hop ranges to different heads in a single Transformer layer, enabling parallel local-to-global message passing before concatenation and output.}
    \label{fig:model_arch}
\end{figure}

\section{Methodology}
\label{sec:methodology}

\subsection{Model Architecture}
\label{subsec:model_arch}

Graph-structured data are distinguished from one- and two-dimensional data by their non-Euclidean connectivity patterns. Let $\mathbf{A}\in\mathbb{R}^{n\times n}$ denote the graph adjacency matrix for a graph with $n$ nodes, where $\mathbf{A}_{ij}=1$ if nodes $i$ and $j$ are connected and $\mathbf{A}_{ij}=0$ otherwise. Although $\mathbf{A}$ encodes critical topological priors, canonical neural architectures such as the Transformer do not natively incorporate such relational inductive biases. A principal objective of this work is therefore to inject graph topology into the Transformer framework in a principled and computationally efficient manner.

As illustrated in Fig. \ref{fig:model_arch}, we introduce X-LogSMask (E\underline{x}plainable Multi-head \underline{Log}arithmic \underline{S}tructural \underline{Mask}), a compact structural bias constructed from $\mathbf{A}$ and integrated additively into the attention logits. X-LogSMask modifies attention computation while preserving the Transformer's core architecture and optimization benefits. Compared to prior Graph Transformer designs, X-LogSMask offers three key advantages:

\begin{enumerate}
  \item It enforces topologically constrained message passing by suppressing attention between non-adjacent nodes via an additive structural mask.
  \item It injects an explicit graph inductive bias directly into the attention matrix, improving the model's sensitivity to graph topology.
  \item By assigning distinct powers of a normalized adjacency to different attention heads, it encodes multi-hop structural information within a single layer, accelerating information propagation and reducing the randomness associated with vanilla multi-head attention.
\end{enumerate}

\subsection{X-LogSMask}

The standard Transformer's all-to-all attention disregards graph topology and can induce excessive, uninformative communication across distant nodes. X-LogSMask remedies this by embedding structural constraints into the attention logits while maintaining the differentiable, end-to-end training behavior of the Transformer.

\subsubsection{Symmetric Normalization of the Adjacency Matrix}

The raw adjacency matrix $\mathbf{A} \in \mathbb{R}^{n \times n}$ inherently suffers from uneven node degree distribution, which can distort graph signal propagation and bias message aggregation toward high-degree nodes. The phenomenon will be strengthened especially in the multi-hop connections. To stabilize propagation and mitigate degree-related distortion, we adopt symmetric normalization with self-loops:
\begin{equation}
\tilde{\mathbf{A}} = \mathbf{D}^{-1/2}(\mathbf{A} + \mathbf{I})\mathbf{D}^{-1/2}
\label{eq:1}
\end{equation}
where $\mathbf{I}$ is the identity matrix and $\mathbf{D}$ is the diagonal degree matrix with $\mathbf{D}_{ii}=\sum_j(\mathbf{A}_{ij}+\mathbf{I}_{ij})$. This normalization balances incoming and outgoing signals and improves numerical stability for subsequent matrix operations.

\subsubsection{Logarithmic Structural Mask}

Directly using $\tilde{\mathbf{A}}$ as an additive attention bias would risk an excessive focus on strong connections while collapsing the influence of weaker, yet potentially informative, links. To compress this dynamic range while preserving the relative ordering of connections, we apply a logarithmic transform. The resulting structural mask $\mathbf{M}$ is defined element-wise as:
\begin{equation}
\mathbf{M}_{ij} = \log(\tilde{\mathbf{A}}_{ij} + \epsilon)
\label{eq:log_mask}
\end{equation}
where $\epsilon$ is a small constant (e.g., $10^{-30}$) to ensure numerical stability for $\tilde{\mathbf{A}}_{ij}=0$.

Let $\mathbf{q}_i^{(r)}$, $\mathbf{k}_j^{(r)}$, and $\mathbf{v}_j^{(r)}$ denote the query, key, and value vectors in the $r$-th attention head, and let $d_h$ be the dimensionality of each head. X-LogSMask is added to the query-key energy before softmax:
\begin{equation}
\alpha_{ij}^{(r)}
=
\frac{
\exp\left(
\frac{
\mathbf{q}_i^{(r)}{\mathbf{k}_j^{(r)}}^{\mathrm{T}}
+
\mathbf{M}_{ij}^{(r)}
}{
\sqrt{d_h}
}
\right)
}{
\sum_{u\in\mathcal{V}}
\exp\left(
\frac{
\mathbf{q}_i^{(r)}{\mathbf{k}_u^{(r)}}^{\mathrm{T}}
+
\mathbf{M}_{iu}^{(r)}
}{
\sqrt{d_h}
}
\right)
}
\end{equation}
where $\mathbf{M}^{(r)}$ denotes the structural mask used by the $r$-th head. The corresponding head output is
\begin{equation}
\mathbf{z}_i^{(r)}
=
\sum_{j\in\mathcal{V}}
\alpha_{ij}^{(r)}\mathbf{v}_j^{(r)}
\end{equation}

This transformation maps bounded structural weights to non-positive log-space biases. Under the attention formulation above, the logarithmic mask acts as a scaled multiplicative structural gate after exponentiation. Specifically, when the $r$-th head uses a structural matrix $\mathbf{S}^{(r)}$ with $\mathbf{M}_{ij}^{(r)}=\log(\mathbf{S}_{ij}^{(r)}+\epsilon)$,
\begin{equation}
\begin{gathered}
\exp\left(
\frac{
\mathbf{q}_i^{(r)}{\mathbf{k}_j^{(r)}}^{\mathrm{T}}
+
\log\left(\mathbf{S}_{ij}^{(r)}+\epsilon\right)
}{
\sqrt{d_h}
}
\right)
\\
=
\exp\left(
\frac{
\mathbf{q}_i^{(r)}{\mathbf{k}_j^{(r)}}^{\mathrm{T}}
}{
\sqrt{d_h}
}
\right)
\left(\mathbf{S}_{ij}^{(r)}+\epsilon\right)^{1/\sqrt{d_h}}
\end{gathered}
\end{equation}
Thus, feature-dependent attention is preserved, while structurally weak or unsupported node interactions receive smaller pre-softmax weights. For absent connections, the small constant $\epsilon$ yields a strong negative bias rather than an exact zero mask, so X-LogSMask implements a near-hard topological constraint in practice.

\subsubsection{Explainable Multi-head Mechanism}

A naive scheme that broadcasts the same first-order structural mask across all attention heads would fail to capture multi-scale topology without deep stacking. To enable per-head specialization, we construct head-specific masks from successive powers of the normalized adjacency. Formally, for $k=1,\dots,h$ define
\begin{align}
\tilde{\mathbf{A}}^k &= \underbrace{\tilde{\mathbf{A}}\cdot\tilde{\mathbf{A}}\cdots\tilde{\mathbf{A}}}_{k\ \text{times}}
\end{align}
and the corresponding log-masks
\begin{equation}
\mathbf{M}^k_{ij} = \log\big(\tilde{\mathbf{A}}^k_{ij} + \epsilon\big)
\end{equation}

We associate the $k$-th head with mask $\mathbf{M}^k$, thereby encouraging the head to attend preferentially to paths of length~$k$. Intuitively, low-order heads emphasize local neighborhoods while higher-order heads capture longer-range dependencies. This structured division eliminates the randomness typical of standard multi-head attention and enables multi-scale graph analysis within a single layer—surpassing traditional message-passing GNNs that require deep architectures to capture long-range dependencies. Empirical results show that our model achieves near-state-of-the-art performance with only 1 layer.

Note that powers of $\tilde{\mathbf{A}}$ remain numerically stable under mild spectral conditions (for example, $\|\tilde{\mathbf{A}}\|_2 \leq 1$ when self-loops are included), and therefore additional normalization is typically unnecessary for modest $k$. By contrast, we observe empirically that an alternative pipeline—raising the raw adjacency matrix $\mathbf{A}+\mathbf{I}$ to the $k$-th power and only then applying a global normalization—produces severe degree imbalance in the resulting multi-hop matrices. Intuitively, this effect arises because $(\mathbf{A}+\mathbf{I})^k$ amplifies high-degree nodes multiplicatively, and a posteriori normalization cannot fully compensate for the resultant heterogeneity. This observation motivates our practical strategy of \emph{normalize first, then raise to the $k$-th power}, which yields more balanced multi-hop structural descriptors for downstream attention masking.

Furthermore, many message-passing GNNs deteriorate as depth increases due to over-smoothing or node homogenization induced by repeated multi-hop aggregation. This issue is absent in our proposed Transformer with X-LogSMask; instead, a problem with a similar origin is shifted to the number of heads in X-LogSMask. High-order adjacency matrices lack strong practical utility.

\section{Experiment}

We evaluated X-LogSMask across node-, edge- and graph-level prediction tasks to test whether structural masking improves Transformer-based graph learning. The evaluation included 20 benchmark datasets and compared two X-LogSMask configurations with message-passing GNNs, Graph Transformer variants and task-specific baselines. We first assessed predictive performance across the three task families, then examined the contribution of individual components through ablation, lightweight one-layer models and attention decomposition.

\subsection{Node Level}
\subsubsection{Experimental Settings}

We evaluated node-level classification on eight benchmark datasets: Cora, Citeseer, Pubmed, Computers, Photo, CS, Physics and WikiCS \cite{sen2008collective,shchur2018pitfalls,mernyei2020wikics}. These datasets cover citation, co-purchase, co-authorship and web-network domains, and span small citation graphs to larger academic networks. For each dataset, we strictly followed the train, validation and test splits used in established benchmark studies \cite{luo2024classic}.

We compared X-LogSMask with three representative message-passing GNNs, GCN, GraphSAGE and GAT \cite{kipf2017semi,hamilton2017inductive,velivckovic2018graph}, and seven Graph Transformer variants, SGFormer, Polynormer, GOAT, NodeFormer, NAGphormer, GraphGPS and Exphormer \cite{wu2023sgformer,deng2024polynormer,kong2023goat,wu2022nodeformer,chen2022nagphormer,rampasek2022graphgps,shirzad2023exphormer}. Baseline results were taken from recent comparative benchmarks under matched evaluation protocols \cite{luo2024classic}.

During training, the small citation datasets showed rapid memorization by the Transformer-based model. In several runs, training accuracy reached 100\% early, while validation and test accuracy remained lower. To reduce this effect, we used random subgraph sampling as training-time augmentation on small node-level datasets. This sampling procedure was not part of X-LogSMask itself and was not used for graph-level tasks.

\subsubsection{Node-level Performance}

On node classification, the full X-LogSMask model achieved the best average rank among all compared methods, with an average rank of 3.3 across eight datasets (Table \ref{tab:node_cls_accuracy_with_avgrank}). It obtained the highest accuracy on Photo, CS, Physics and WikiCS, and the second-highest accuracy on Computers. These gains were most evident on the larger node-classification benchmarks, where structural masking improved the Transformer's ability to aggregate graph-dependent information.

\begin{table}[htbp]
\centering
\small
\setlength{\tabcolsep}{2pt}
\renewcommand{\arraystretch}{1.2}
\caption{Experimental results on node-level tasks. The top \textbf{1\textsuperscript{st}} and \underline{$2^{nd}$} are highlighted.}
\label{tab:node_cls_accuracy_with_avgrank}
\begin{tabularx}{\textwidth}{>{\raggedright\arraybackslash}p{16em} *{9}{>{\centering\arraybackslash}X}@{}}
\toprule
\multicolumn{1}{c}{} & Cora & Citeseer & Pubmed & \makecell{Comput-\\ers} & Photo & CS & Physics & WikiCS & \makecell{Avg.\\Rk.} \\
\cmidrule(lr){2-2} \cmidrule(lr){3-3} \cmidrule(lr){4-4} \cmidrule(lr){5-5} \cmidrule(lr){6-6} \cmidrule(lr){7-7} \cmidrule(lr){8-8} \cmidrule(lr){9-9} \cmidrule(l){10-10}
\multicolumn{1}{c}{Metric} & Acc. & Acc. & Acc. & Acc. & Acc. & Acc. & Acc. & Acc. & - \\
\midrule
\multicolumn{10}{c}{\textit{MP-GNNs}} \\
\midrule
GCN [Kipf and Welling, 2017] & 81.60 & 71.60 & 78.80 & 89.65 & 92.70 & 92.92 & 96.18 & 77.47 & 10.6 \\
GraphSAGE [Hamilton et al., 2017] & 82.68 & 71.93 & 79.41 & 91.20 & 94.59 & 93.91 & 96.49 & 74.77 & 8.1 \\
GAT [Veličković et al., 2018] & 83.00 & 72.50 & 79.00 & 90.78 & 93.87 & 93.61 & 96.17 & 76.91 & 8.6 \\
\midrule
\multicolumn{10}{c}{\textit{Graph Transformer variants}} \\
\midrule
GraphGPS [Rampášek et al., 2022] & 82.84 & \textbf{72.73} & \underline{79.94} & 91.19 & 95.06 & 93.93 & 97.12 & 78.66 & 5.1 \\
NAGphormer [Chen et al., 2023] & 82.12 & 71.47 & 79.73 & 91.22 & 95.49 & 95.75 & 97.34 & 77.16 & 5.9 \\
NodeFormer [Wu et al., 2022] & 82.20 & 72.50 & 79.90 & 86.98 & 93.46 & 95.64 & 96.45 & 74.73 & 7.6 \\
GOAT [Kong et al., 2023] & 83.18 & 71.99 & 79.13 & 90.96 & 92.96 & 94.21 & 96.24 & 77.00 & 8.1 \\
Exphormer [Shirzad et al., 2023] & 82.77 & 71.63 & 79.46 & 91.47 & 95.35 & 94.93 & 96.89 & 78.54 & 5.9 \\
SGFormer [Wu et al., 2023] & \textbf{84.50} & \underline{72.60} & \textbf{80.30} & 91.99 & 95.10 & 94.78 & 96.60 & 73.46 & 4.9 \\
Polynormer [Deng et al., 2024] & \underline{83.25} & 72.31 & 79.24 & \textbf{93.68} & \underline{96.46} & 95.53 & 97.27 & \underline{80.10} & \underline{3.8} \\
\midrule
\makecell[l]{Trans. with X-LogSMask (1-layer)} & 80.80 & 71.00 & 79.20 & 91.21 & 95.69 & \underline{96.56} & \underline{97.65} & \textbf{80.36} & 5.9 \\
\makecell[l]{Trans. with X-LogSMask (ours)} & 82.70 & 71.70 & 79.60 & \underline{92.01} & \textbf{96.86} & \textbf{96.62} & \textbf{97.68} & \textbf{80.36} & \textbf{3.3} \\
\bottomrule
\end{tabularx}
\end{table}

The one-layer X-LogSMask model remained competitive despite its reduced depth. It ranked within the top three on Photo, CS, Physics and WikiCS, and achieved an average rank of 5.9. This result showed that head-wise multi-hop structural masks can provide useful receptive-field expansion without requiring deep Transformer stacks.

Performance was weaker on the smaller citation datasets. On Cora, Citeseer and Pubmed, X-LogSMask did not exceed the strongest competing Graph Transformer baselines. The corresponding training curves showed early saturation of training accuracy, indicating that overfitting limited performance in these low-data regimes. This limitation is consistent with the need for random subgraph augmentation during node-level training.

% \begin{figure}[!t]
%     \centering
%     \includegraphics[width=\linewidth]{exp_result_1.png}
%     \caption{Integrated benchmark and component analysis for X-LogSMask. (a) Average rank across node classification, edge regression, link prediction, and graph-level tasks; lower values indicate stronger performance. (b) Cross-task effect-size distribution relative to the strongest non-X-LogSMask baseline. (c) Accuracy margin for node classification across eight datasets. (d) Performance margin for graph-level benchmarks across seven datasets. (e) MRR gain for link prediction. (f) Relative MAE/RMSE reduction for edge regression, where higher values indicate larger error reductions. (g) Node-classification ablation results after removing symmetric normalization, LogSMask, or the explainable multi-head mechanism. (h) Mean and maximum accuracy drops caused by component removal. (i) Performance difference between the 1-layer and full X-LogSMask configurations across task families; positive values favour the 1-layer model, with edge-regression signs converted so that higher values are better.}
%     \label{fig:exp_result_1}
% \end{figure}

\subsection{Edge Level}
\subsubsection{Experimental Settings}

We evaluated edge-level performance in two settings: link prediction and edge regression. For link prediction, we used Cora and Citeseer \cite{sen2008collective}, with mean reciprocal rank (MRR) as the evaluation metric. For edge regression, we used three temporal edge-regression datasets from the TER benchmark, epic-games-plr, air-traffic-2015-rlr and air-traffic-2019-rlr \cite{Ozmen2024BenchmarkingER}. These datasets cover game-rating prediction and flight-delay estimation.

For link prediction, X-LogSMask was compared with three heuristic baselines, common neighbours (CN), Adamic-Adar (AA) and resource allocation (RA), and eight neural baselines, GCN, SAGE, SEAL, Neo-GNN, BUDDY, NCN, NCNC and LPFormer \cite{kipf2017semi,hamilton2017inductive,zhang2018link,yun2021neo,zhao2023buddy,liu2024ncn,Shomer_2024}. For edge regression, we compared against moving-average baselines and static message-passing GNNs, including eGCN, eGSage, eGAT, eGTransf and their rich-feature variants.

Dense attention has $O(n^2)$ memory and computational complexity. We therefore restricted these experiments to small- and medium-scale graphs, and did not use the edge-level benchmarks to assess ultra-large-graph scalability. Link-prediction datasets were randomly split into training, validation and test sets in an 85:5:10 ratio \cite{Shomer_2024}. Edge-regression datasets were split in a 7:1:2 ratio.

\subsubsection{Edge-level Performance}

For edge regression, X-LogSMask also ranked first across all three datasets (Table \ref{tab:edge_regress_results}). Both configurations achieved an average rank of 1.0 and occupied the best or second-best position in every MAE and RMSE column. On epic-games-plr, X-LogSMask reduced MAE to 0.0149 and RMSE to 0.0416, compared with 0.1062 and 0.1788 for the strongest baseline. On the two air-traffic datasets, the one-layer and full models again produced nearly identical error values, showing that the structural mask provided most of the edge-regression gain independently of model depth.

\begin{table}[htbp]
  \centering
  \caption{Experimental results on edge-level regression tasks. The top \textbf{1\textsuperscript{st}} and \underline{$2^{nd}$} are highlighted.}
  \label{tab:edge_regress_results}
  \begin{tabularx}{\textwidth}{p{14em} *{7}{>{\centering\arraybackslash}X}@{}}
    \toprule
    \multicolumn{1}{c}{} & \multicolumn{2}{c}{epic-games-plr} & \multicolumn{2}{c}{air-traffic-2019-rlr} & \multicolumn{2}{c}{air-traffic-2015-rlr} & \multicolumn{1}{c}{Avg. Rk.} \\ 
    \cmidrule(lr){2-3} \cmidrule(lr){4-5} \cmidrule(lr){6-7} \cmidrule(l){8-8} 
    Method & MAE & RMSE & MAE & RMSE & MAE & RMSE & - \\ 
    \midrule
    \multicolumn{8}{c}{\textit{Non-parametric Methods}} \\ 
    \midrule
    MA-all (K=10)        & 0.1322 & 0.1954 & 0.4510 & 0.5931 & 0.1877 & 0.3005 & 12.3 \\
    MA-src (K=10)        & 0.1480 & 0.2499 & 0.4421 & 0.5833 & 0.1864 & 0.2993 & 11.7 \\
    MA-dst (K=10)        & 0.2951 & 0.4341 & 0.4444 & 0.5867 & 0.1848 & 0.2970 & 12.0 \\
    \midrule
    \multicolumn{8}{c}{\textit{Graph Neural Networks}} \\ 
    \midrule
    eGCN                 & 0.1178 & 0.1883 & 0.3983 & 0.5467 & 0.1673 & 0.2965 & \underline{4.3} \\
    eGCN-rich            & \underline{0.1062} & \underline{0.1788} & 0.4151 & 0.5593 & 0.1727 & 0.2975 & 6.7 \\
    eGSage               & 0.1205 & 0.1894 & 0.4021 & 0.5502 & 0.1727 & 0.2975 & 8.0 \\
    eGSage-rich          & 0.1191 & 0.1883 & 0.4028 & 0.5516 & 0.1667 & 0.2933 & 6.0 \\
    eGAT                 & 0.1180 & 0.1883 & 0.4252 & 0.5699 & 0.1665 & 0.2944 & 6.0 \\
    eGAT-rich            & 0.1186 & 0.1888 & 0.3995 & 0.5474 & 0.1669 & 0.2942 & 5.0 \\
    eGTransf             & 0.1191 & 0.1894 & 0.4174 & 0.5597 & 0.1678 & 0.2964 & 8.0 \\
    eGTransf-rich        & 0.1190 & 0.1887 & 0.4028 & 0.5507 & 0.1689 & 0.2974 & 7.0 \\
    \midrule
    Trans. with X-LogSMask (1-layer) & \textbf{0.0149} & \textbf{0.0416} & \underline{0.1066} & \textbf{0.1651} & \underline{0.0996} & \textbf{0.1588} & \textbf{1.0} \\
    Trans. with X-LogSMask (ours)    & \textbf{0.0149} & \textbf{0.0416} & \textbf{0.1065} & \underline{0.1654} & \textbf{0.0989} & \underline{0.1591} & \textbf{1.0} \\
    \bottomrule
  \end{tabularx}
\end{table}

For link prediction, both the one-layer and full X-LogSMask models achieved the best performance on Cora and Citeseer (Table \ref{tab:link_pred_results}). The models obtained MRR scores of 59.9 on Cora and 71.5 on Citeseer, with an average rank of 1.0. Compared with the strongest non-X-LogSMask baseline, LPFormer, this corresponded to gains of 20.48 MRR points on Cora and 6.08 MRR points on Citeseer. The identical performance of the one-layer and full configurations indicated that edge-level structural signals were captured effectively without additional Transformer depth in this setting.

\begin{table}[t]
\centering
\caption{Experimental results on edge-level link prediction tasks. The top \textbf{1\textsuperscript{st}} and \underline{$2^{nd}$} are highlighted.}
\label{tab:link_pred_results}
\begin{tabularx}{\textwidth}{p{14em} *{3}{>{\centering\arraybackslash}X}@{}}
\toprule
\multicolumn{1}{c}{} & Cora & Citeseer & Avg. Rk. \\
\cmidrule(lr){2-2} \cmidrule(lr){3-3} \cmidrule(l){4-4}
\centering Metric & MRR & MRR & - \\
\midrule
\multicolumn{4}{c}{\textit{Heuristic Methods}} \\
\midrule
CN & 20.99 & 28.34 & 12.5 \\
AA & 31.87 & 29.37 & 9.0 \\
RA & 30.79 & 27.61 & 10.5 \\
\midrule
\multicolumn{4}{c}{\textit{Graph Neural Networks}} \\
\midrule
GCN \cite{kipf2017semi} & 32.50 & 50.01 & 7.0 \\
SAGE \cite{hamilton2017inductive} & 37.83 & 47.84 & 6.5 \\
SEAL \cite{zhang2018link} & 26.69 & 39.36 & 10.0 \\
Neo-GNN \cite{yun2021neo} & 22.65 & 53.97 & 9.5 \\
BUDDY \cite{zhao2023buddy} & 26.40 & 59.48 & 8.0 \\
NCN \cite{liu2024ncn} & 32.93 & 54.97 & 5.5 \\
NCNC \cite{liu2024ncn} & 29.01 & 64.03 & 6.5 \\
LPFormer \cite{Shomer_2024} & \underline{39.42} & \underline{65.42} & \underline{3.0} \\
\midrule
Trans. with X-LogSMask (1-layer) & \textbf{59.9} & \textbf{71.5} & \textbf{1.0} \\
Trans. with X-LogSMask (ours) & \textbf{59.9} & \textbf{71.5} & \textbf{1.0} \\
\bottomrule
\end{tabularx}
\end{table}

\subsection{Graph Level}

\subsubsection{Experimental Settings}

We evaluated graph-level prediction on seven benchmark datasets: NCI1, D\&D, PROTEINS, MUTAG, COLLAB, IMDB-B and MOLHIV \cite{morris2020tudataset,hu2020ogb}. The first six datasets were taken from TU-Dataset and cover molecular, protein and social-network classification tasks. MOLHIV was taken from the Open Graph Benchmark and was evaluated using AUROC. For TU-Dataset benchmarks, we used random training, validation and test splits in an 8:1:1 ratio. For MOLHIV, we used the prescribed benchmark split.

We compared X-LogSMask with four message-passing GNNs, GCN, GAT, GIN and GatedGCN \cite{kipf2017semi,velivckovic2018graph,xu2019gin,li2016gated}, and seven Graph Transformer variants, GT, SAN, Graphormer, GraphTrans, GMT, SAT and GraphGPS \cite{dwivedi2021gt,kreuzer2021san,ying2021graphormer,wu2021graphtrans,beak2021gmt,chen2022sat,rampasek2022graphgps}. Baseline values were taken from prior benchmark studies to maintain consistency with established evaluation protocols \cite{liu2024gradformer,beak2021gmt}.

\subsubsection{Graph-level Performance}

The full X-LogSMask model achieved the best average rank on graph-level benchmarks, with an average rank of 1.7 across seven datasets (Table \ref{tab:graph_cls_reg_results}). It ranked first on D\&D, PROTEINS, MUTAG and MOLHIV, and second on IMDB-B. The largest improvement was observed on PROTEINS, where X-LogSMask reached 80.63\% accuracy, compared with 75.77\% for the strongest non-X-LogSMask baseline. On D\&D and MUTAG, the model achieved 81.20\% and 88.89\% accuracy, exceeding the best competing results by 2.48 and 1.67 percentage points, respectively.

\begin{table}[htbp]
  \centering
  \caption{Experimental results on graph-level tasks. The top \textbf{1\textsuperscript{st}} and \underline{$2^{nd}$} are highlighted.}
  \label{tab:graph_cls_reg_results}
  \begin{tabularx}{\textwidth}{p{16em} *{8}{>{\centering\arraybackslash}X}@{}}
    \toprule
    \multicolumn{1}{c}{} & NCI1 & D\&D & PROTE. & MUTAG & COLLAB & IMDB-B & MOLHIV & Avg. Rk. \\
    \cmidrule(lr){2-2} \cmidrule(lr){3-3} \cmidrule(lr){4-4} \cmidrule(lr){5-5} \cmidrule(lr){6-6} \cmidrule(lr){7-7} \cmidrule(l){8-8} \cmidrule(l){9-9}
    \centering Metric & Acc. & Acc. & Acc. & Acc. & Acc. & Acc. & AUROC. & - \\
    \midrule
    \multicolumn{9}{c}{\textit{MP-GNNs}} \\
    \midrule
    GCN [Kipf and Welling, 2017] & 79.68 & 72.05 & 71.70 & 73.40 & 71.92 & 74.30 & 75.99 & 10.0 \\
    GAT [Velicković et al., 2018] & 79.88 & - & 72.00 & 73.90 & 75.80 & 74.70 & - & 10.4 \\
    GIN [Xu et al., 2019b] & 81.70 & 70.79 & 73.76 & 84.50 & 73.32 & 75.10 & 77.07 & 6.7 \\
    GatedGCN [Bresson and Laurent, 2017] & 81.17 & - & 74.65 & 85.00 & 80.70 & 73.20 & - & 6.8 \\
    \midrule
    \multicolumn{9}{c}{\textit{Graph Transformer variants}} \\
    \midrule
    GT [Dwivedi and Bresson, 2021] & 80.15 & - & 73.94 & 83.90 & 79.63 & 73.10 & - & 9.2 \\
    SAN [Kreuzer et al., 2021] & 80.50 & - & 74.11 & 78.80 & 79.42 & 72.10 & 77.85 & 8.8 \\
    Graphormer [Ying et al., 2021] & 81.44 & - & 75.29 & 80.52 & \textbf{81.80} & 73.40 & 74.55 & 6.2 \\
    GraphTrans [Cai and Lam, 2019] & \underline{82.60} & - & 75.18 & \underline{87.22} & 79.81 & 74.50 & 76.33 & 5.0 \\
    GMT [Baek et al., 2021] & - & \underline{78.72} & 75.09 & 83.44 & 80.74 & 73.48 & - & 6.0 \\
    SAT [Chen et al., 2022] & 80.69 & - & 73.32 & 80.50 & 80.05 & 75.90 & - & 7.8 \\
    GraphGPS [Ramp\'{a}\v{s}ek et al., 2022] & \textbf{84.21} & - & \underline{75.77} & 85.00 & \underline{81.40} & \textbf{77.40} & \underline{78.80} & \underline{2.0} \\
    \midrule
    Trans. with X-LogSMask (1-layer) & 81.27 & \textbf{81.20} & 75.68 & \textbf{88.89} & 79.00 & 76.00 & 77.68 & 4.0 \\
    Trans. with X-LogSMask (ours) & 82.24 & \textbf{81.20} & \textbf{80.63} & \textbf{88.89} & 80.80 & \underline{77.00} & \textbf{78.91} & \textbf{1.7} \\
    \bottomrule
  \end{tabularx}
\end{table}

The gains were not uniform across all datasets. On NCI1, X-LogSMask achieved 82.24\% accuracy, ranking behind GraphGPS and GraphTrans. On COLLAB, it reached 80.80\% accuracy, below Graphormer and GraphGPS. These results show that X-LogSMask provided strong overall graph-level performance, but its advantage was dataset-dependent.

The 1-layer lightweight configuration exhibits remarkable efficiency, matching the full-version's top performance on D\&D and MUTAG while maintaining competitiveness on other datasets. It outperforms most complex Graph Transformer variants and traditional MP-GNNs, demonstrating that X-LogSMask enables reduced model complexity without performance degradation.

\subsection{Further Discussion}

\subsubsection{Ablation}

To quantify the contribution of each core component, we perform an ablation study
on node-level classification benchmarks. Starting from the full model, we removed symmetric normalization, the logarithmic structural mask and the explainable multi-head mechanism one at a time. The results are summarized in Table \ref{tab:ablation_node_cls}.

\begin{table}[htbp]
  \begin{threeparttable}
  \centering
  \caption{Ablation Study on Key Components of the proposed Model.}
  \label{tab:ablation_node_cls}
  \footnotesize\setlength{\tabcolsep}{3pt}%
    \begin{tabular*}{\linewidth}{@{\extracolsep{\fill}} ccccccccccc @{}}
    \toprule
    \multicolumn{3}{c}{Components} &
    \multirow{2}{*}{Cora} & \multirow{2}{*}{Citeseer} & \multirow{2}{*}{Pubmed} & 
    \multirow{2}{*}{Computers} & \multirow{2}{*}{Photo} & \multirow{2}{*}{CS} & 
    \multirow{2}{*}{Physics} & \multirow{2}{*}{WikiCS} \\
    \cmidrule{1-3}    
    \textit{Sym. Norm.} & \textit{LogSMask} & \textit{Expl. MH} & & & & & & & & \\
    \midrule
    \checkmark & \checkmark & \checkmark & 82.70 & 71.70 & 79.60 & 91.57 & 96.86 & 96.62 & 97.68 & 80.36 \\
    $\times$ & \checkmark & \checkmark & 77.00\textsubscript{\textcolor{blue}{↓5.70}} & 7.70\textsubscript{\textcolor{blue}{↓64.00}} & 74.50\textsubscript{\textcolor{blue}{↓5.10}} & 3.27\textsubscript{\textcolor{blue}{↓88.30}} & 4.84\textsubscript{\textcolor{blue}{↓95.00}} & 95.69\textsubscript{\textcolor{blue}{↓0.93}} & 97.16\textsubscript{\textcolor{blue}{↓0.52}} & 4.01\textsubscript{\textcolor{blue}{↓76.35}} \\
    \checkmark & $\times$ & \checkmark & 58.80\textsubscript{\textcolor{blue}{↓23.90}} & 57.80\textsubscript{\textcolor{blue}{↓13.90}} & 72.40\textsubscript{\textcolor{blue}{↓7.20}} & 76.74\textsubscript{\textcolor{blue}{↓14.83}} & 87.06\textsubscript{\textcolor{blue}{↓9.80}} & 95.20\textsubscript{\textcolor{blue}{↓1.42}} & 96.35\textsubscript{\textcolor{blue}{↓1.33}} & 71.73\textsubscript{\textcolor{blue}{↓8.63}} \\
    \checkmark & \checkmark & $\times$ & 77.40\textsubscript{\textcolor{blue}{↓5.30}} & 68.00\textsubscript{\textcolor{blue}{↓3.70}} & 76.20\textsubscript{\textcolor{blue}{↓3.40}} & 92.01\textsubscript{\textcolor{red}{↑0.44}} & 95.68\textsubscript{\textcolor{blue}{↓1.18}} & 95.80\textsubscript{\textcolor{blue}{↓0.82}} & 97.68\textsubscript{\textcolor{blue}{=0.00}} & 79.42\textsubscript{\textcolor{blue}{↓0.94}} \\
    \bottomrule
    \end{tabular*}%
    \begin{tablenotes}[para,flushleft]
      \small \textit{Note:} \textcolor{blue}{↓} = performance degradation, \textcolor{red}{↑} = performance improvement.
    \end{tablenotes}
  \end{threeparttable}
\end{table}%

Removing symmetric normalization caused the largest performance collapse. Accuracy dropped by up to 95.00\% on Photo and 88.30\% on Computers, showing that degree normalization is essential for stable propagation and for preventing multi-hop weights from becoming overly biased toward high-degree nodes. Removing the logarithmic structural mask also caused severe degradation, including a 23.90\% drop on Cora, indicating that the structural mask is the main source of graph-specific inductive bias. By contrast, removing the explainable multi-head mechanism produced only modest changes, and occasionally a small gain, such as the 0.44\% increase on Computers. This suggests that the multi-head design mainly supports structural specialization and interpretability, rather than serving as the primary source of predictive capacity.

\subsubsection{Interpretability of X-LogSMask}
Self-attention energy matrices provide a direct view of inter-node message passing. To inspect how X-LogSMask decomposes attention, we visualized the raw energy term, the structural term and their sum across two layers and four heads of a pre-trained model on PROTEINS, as shown in Fig. \ref{fig:exp_result_2}.

The raw energy term mainly captured node-attribute similarity, whereas X-LogSMask injected edge-level structure. Their additive combination therefore coupled content similarity with topological bias in a single attention score. The heads also showed distinct hop preferences, indicating that different heads specialized in different structural radii. Comparing the two layers further showed a clear progression: the first layer mixed node and edge information, whereas the final layer before classification became more GNN-like, with edge structure exerting a stronger influence on node representations.

\begin{figure}[H]
    \centering
    \includegraphics[width=0.85\linewidth]{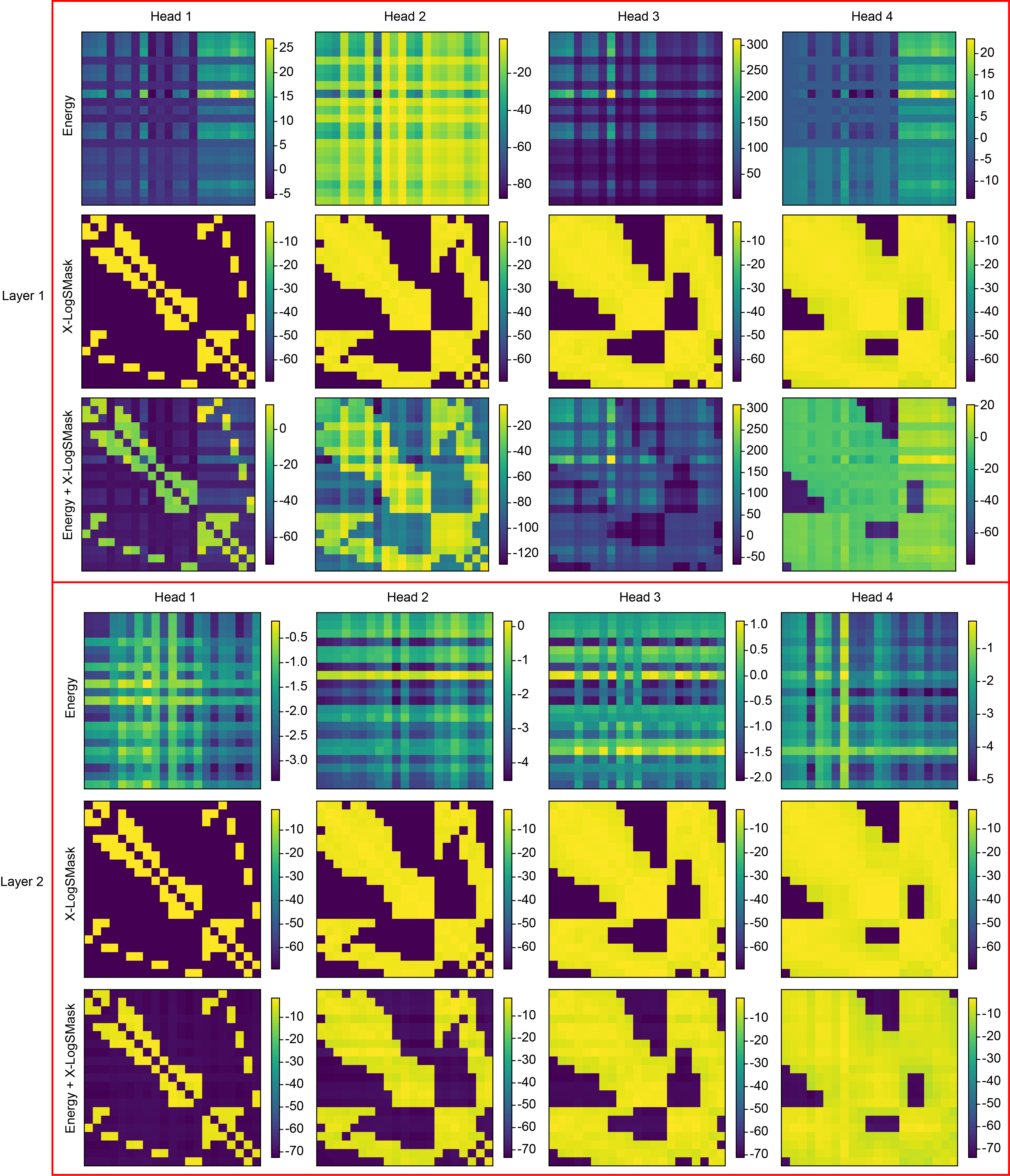}
    \caption{Layer-wise attention decomposition in a pre-trained 2-layer, 4-head X-LogSMask Transformer on PROTEINS. For each layer and head, we compare the raw attention energy, the X-LogSMask structural term, and their additive combination. The comparison shows that the energy term mainly captures node-attribute similarity, whereas X-LogSMask injects edge-level structure; their sum therefore couples node-content information and structural bias within a single attention score. Head-specific patterns further indicate that different heads focus on different hop ranges, allowing one layer to approximate multi-hop connectivity without deep stacking. Comparing the two layers shows a clear progression: the first layer still mixes node and edge information, whereas the final layer before classification becomes more GNN-like, with edge structure exerting a stronger influence on node representations.}
    \label{fig:exp_result_2}
\end{figure}

\subsubsection{Lightweight 1-layer Solution}
\label{sec:1-layer comparison}

The head-wise structural decomposition of X-LogSMask allows a single Transformer layer to capture multi-hop graph context. Across the benchmark suite, the 1-layer model remained competitive with deeper baselines and matched the full model on several edge-level tasks, as shown in Tables \ref{tab:node_cls_accuracy_with_avgrank}, \ref{tab:edge_regress_results}, \ref{tab:link_pred_results}, \ref{tab:graph_cls_reg_results}, and Fig. \ref{fig:light1l}.

\begin{figure}[H]
    \centering
    \includegraphics[width=0.8\linewidth]{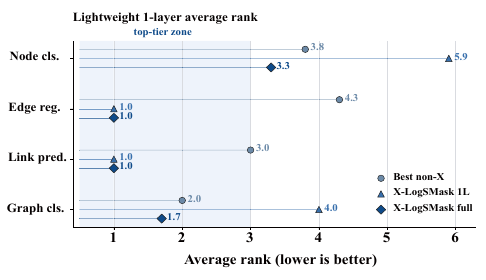}
    \caption{Average ranks of X-LogSMask variants across task families. The one-layer configuration is compared with the full X-LogSMask and the best non-X baseline, with lower ranks indicating better relative performance and the shaded region denoting the top-tier rank zone.}
    \label{fig:light1l}
\end{figure}

This result shows that structural masking can partially substitute for depth in graph learning. It does not remove the quadratic cost of dense attention within a layer, but it reduces the need for repeated stacking and therefore lowers the overall depth-related cost. We note that the 1-layer design trades breadth for depth: careful selection of head counts and head-specific structural orders is important to maximize the effective receptive field without introducing redundancy. The trade-off is also clear: one-layer X-LogSMask is often strong, but the full model still performs better on several graph-level datasets.

\section{Conclusion}

We present X-LogSMask, a concise and interpretable modification to the self-attention matrix. It embeds graph topology as an additive structural bias while preserving the discriminative power of multiplicative attention, thereby enabling Transformers to natively handle graph-structured data. Through comprehensive experiments on node-, edge-, and graph-level benchmarks, X-LogSMask consistently improves performance over strong baselines and establishes new state-of-the-art results. Moreover, we present a lightweight 1-layer solution for large-scale graphs and validate the model’s interpretability. The core insight of introducing structured, multi-scale masks into attention can be adapted to other structured modalities, suggesting a general pathway for other domains.

\section*{Data availability}

All datasets used in this study are publicly available benchmark datasets. Cora, Citeseer, Pubmed, Computers, Photo, CS, Physics, and WikiCS can be accessed through the PyTorch Geometric dataset loaders for Planetoid, Amazon, Coauthor, and WikiCS datasets at \url{https://pytorch-geometric.readthedocs.io/en/latest/modules/datasets.html}. The edge-regression datasets are available from the TER benchmark repositories on Hugging Face: epic-games-plr at \url{https://huggingface.co/datasets/cash-app-inc/epic-games-plr}, air-traffic-2019-rlr at \url{https://huggingface.co/datasets/cash-app-inc/air-traffic-2019-rlr}, and air-traffic-2015-rlr at \url{https://huggingface.co/datasets/cash-app-inc/air-traffic-2015-rlr}. NCI1, D\&D, PROTEINS, MUTAG, COLLAB, and IMDB-B are available from TUDataset at \url{https://chrsmrrs.github.io/datasets/docs/datasets/}. MOLHIV is available from the Open Graph Benchmark graph property prediction collection at \url{https://ogb.stanford.edu/docs/graphprop/#ogbg-mol}. CIFAR-10 is available from the original CIFAR-10 data source at \url{https://www.cs.toronto.edu/~kriz/cifar.html}.

\section*{Code availability}

The source code is available at \url{https://github.com/LiLeyan-0120/X-LogSMask}.

{\appendix
\section*{Datasets}

We utilize a total of 20 datasets to evaluate model performance across node-, edge-, and graph-level tasks. These graph-learning experiments were executed on A100 GPUs, whereas the auxiliary CIFAR-10 Vision Transformer experiments were conducted on a single A5000 GPU.

\begin{itemize}
    \item Node-level: CS, Physics (OGB [Hu et al., 2020]); Computers, Photo (Benchmarking GNN [Dwivedi et al., 2023]); Cora, Citeseer, Pubmed [Sen et al., 2008]; WikiCS [Zitnik \& Leskovec, 2017].
    \item Edge-level: Cora, Citeseer [Sen et al., 2008]; epic-games-plr, air-traffic-2019-rlr, air-traffic-2015-rlr [Ozmen et al., 2024].
    \item Graph-level: NCI1, D\&D, PROTEINS, MUTAG, COLLAB, IMDB-B (TU-Dataset [Morris et al., 2020]); MOLHIV (OGB [Hu et al., 2020]).
\end{itemize}

These datasets span citation networks, social networks, molecular graphs, biological networks, traffic networks, game review platforms, and NFT transaction systems, with varied scales to ensure comprehensive evaluation of generalization across graph types and tasks. Detailed statistics are in Tables \ref{tab:node_datasets}, \ref{tab:edge_datasets}, and \ref{tab:graph_datasets}.

\begin{table}[htbp]
    \centering
    \caption{Datasets for Node-Level Tasks}
    \begin{tabular}{cccccc}
        \toprule
        Dataset & \# Nodes & \# Edges & Predict level & Predict task & Metric \\
        \midrule
        Cora & 2,708 & 5,278 & node & 7-class classif. & Accuracy \\
        Citeseer & 3,327 & 4,732 & node & 6-class classif. & Accuracy \\
        Pubmed & 19,717 & 44,324 & node & 3-class classif. & Accuracy \\
        Computers & 13,752 & 245,861 & node & 10-class classif. & Accuracy \\
        Photo & 7,650 & 119,081 & node & 8-class classif. & Accuracy \\
        CS & 18,333 & 81,894 & node & 15-class classif. & Accuracy \\
        Physics & 34,493 & 247,962 & node & 5-class classif. & Accuracy \\
        WikiCS & 11,701 & 216,123 & node & 10-class classif. & Accuracy \\
        \bottomrule
    \end{tabular}
    \label{tab:node_datasets}
\end{table}

\begin{table}[htbp]
  \centering
  \caption{Datasets for Edge-Level Tasks}
  \begin{tabular}{cccccc}
    \toprule
    Dataset & \# Nodes & \# Edges & Predict level & Predict task & Metric \\
    \midrule
    Cora & 2,708 & 5,278 & edge & link prediction & MRR \\
    Citeseer & 3,327 & 4,732 & edge & link prediction & MRR \\
    epic-games-plr & 1,156 & 17,584 & edge & edge regression & MAE, RMSE \\
    air-traffic-2019-rlr & 274 & 484,551 & edge & edge regression & MAE, RMSE \\
    air-traffic-2015-rlr & 257 & 5,138,263 & edge & edge regression & MAE, RMSE \\
    \bottomrule
  \end{tabular}
  \label{tab:edge_datasets}
\end{table}

\begin{table}[htbp]
    \centering
    \caption{Datasets for Graph-Level Tasks}
    \begin{tabular}{ccccccc}
        \toprule
        Dataset & \# Graphs & \# Nodes (avg) & \# Edges (avg) & Predict level & Predict task & Metric \\
        \midrule
        NCI1 & 4,110 & 29.8 & 32.3 & graph & 2-class classif. & Accuracy \\
        D\&D & 1,178 & 284.3 & 715.6 & graph & 2-class classif. & Accuracy \\
        PROTEINS & 1,113 & 39.1 & 72.8 & graph & 2-class classif. & Accuracy \\
        MUTAG & 188 & 17.9 & 19.7 & graph & 2-class classif. & Accuracy \\
        COLLAB & 5,000 & 74.5 & 2457.7 & graph & 3-class classif. & Accuracy \\
        IMDB-B & 1,000 & 19.8 & 96.5 & graph & 2-class classif. & Accuracy \\
        MOLHIV & 41,127 & 25.5 & 27.5 & graph & 2-class classif. & AUROC. \\
        \bottomrule
    \end{tabular}
    \label{tab:graph_datasets}
\end{table}

\begin{table}[htbp]
\centering
\caption{Hyper-parameter configurations for the CIFAR-10 image-classification experiments. All runs share identical data splits and random seeds.}
\label{tab:vit_params}
\footnotesize\setlength{\tabcolsep}{4pt}
\begin{tabular}{cccccccccccc}
\toprule
Run & Model & Patch size & Embed. dim & Pos. Encoding & Layers & Heads & Dropout & Batch size & LR & Optimiser \\
\midrule
\multirow{2}{*}{T1} & ViT            & $2\!\times\!2$ & 256 & \checkmark & 8 & 8 & 0.1 & 128 & $1\!\times\!10^{-3}$ & AdamW \\
                    & ViT + X-LogSMask  & $2\!\times\!2$ & 256 & $\times$ & 8 & 8 & 0.1 & 128 & $1\!\times\!10^{-3}$ & AdamW \\
\midrule
\multirow{2}{*}{T2} & ViT            & $2\!\times\!2$ & 256 & \checkmark & 6 & 8 & 0.1 & 128 & $1\!\times\!10^{-3}$ & AdamW \\
                    & ViT + X-LogSMask  & $2\!\times\!2$ & 256 & $\times$ & 6 & 8 & 0.1 & 128 & $1\!\times\!10^{-3}$ & AdamW \\
\midrule
\multirow{2}{*}{T3} & ViT            & $2\!\times\!2$ & 256 & \checkmark & 6 & 4 & 0.1 & 128 & $1\!\times\!10^{-3}$ & AdamW \\
                    & ViT + X-LogSMask  & $2\!\times\!2$ & 256 & $\times$ & 6 & 4 & 0.1 & 128 & $1\!\times\!10^{-3}$ & AdamW \\
\bottomrule
\end{tabular}
\end{table}

\section*{Extending X-LogSMask to Vision Transformers}

To test whether the structural mask transfers beyond graph data, we integrated X-LogSMask into a standard Vision Transformer and evaluated it on CIFAR-10. To isolate the effect of the mask, we kept the training setup fixed, without learning-rate warmup, extra augmentation or extensive hyperparameter tuning. The two models were trained under identical settings on a single A5000 GPU, and the hyperparameters are listed in Table \ref{tab:vit_params}.

X-LogSMask conferred a modest yet consistent improvement in classification accuracy relative to the baseline ViT. Across the three controlled runs, X-LogSMask improved ViT accuracy from 82.18\% to 86.07\% in T1, from 80.96\% to 85.08\% in T2, and from 81.44\% to 84.15\% in T3, corresponding to gains of 3.89, 4.12 and 2.71 percentage points. These results show that the structural mask can also improve a grid-based Transformer, although the main focus of this paper remains graph-structured data.

\begin{table}[htbp]
  \centering
  \caption{Image-classification accuracy on CIFAR-10.}
  \label{tab:cifar_ablation}
  \begin{tabular*}{\linewidth}{@{\extracolsep{\fill}} >{\centering\arraybackslash}m{1.8cm}
                             >{\centering\arraybackslash}m{2.4cm}
                             >{\centering\arraybackslash}m{2.5cm} @{}}
    \toprule
    Run & Model & Accuracy/\% \\
    \midrule
    \multirow{2}{*}{T1} & ViT           & 82.18 \\
                        & ViT + X-LogSMask & 86.07\textsubscript{\textcolor{red}{↑3.89}} \\
    \addlinespace
    \multirow{2}{*}{T2} & ViT           & 80.96 \\
                        & ViT + X-LogSMask & 85.08\textsubscript{\textcolor{red}{↑4.12}} \\
    \addlinespace
    \multirow{2}{*}{T3} & ViT           & 81.44 \\
                        & ViT + X-LogSMask & 84.15\textsubscript{\textcolor{red}{↑2.71}} \\
    \bottomrule
  \end{tabular*}
  \begin{tablenotes}[para,flushleft]
    \small \textit{Note:} \textcolor{red}{↑} indicates improvement brought by X-LogSMask.
  \end{tablenotes}
\end{table}
}

\bibliographystyle{unsrtnat}
\bibliography{citation}

@article{vaswani2017attention,
  title={Attention is all you need},
  author={Vaswani, Ashish and Shazeer, Noam and Parmar, Niki and Uszkoreit, Jakob and Jones, Llion and Gomez, Aidan N and Kaiser, {\L}ukasz and Polosukhin, Illia},
  journal={Advances in Neural Information Processing Systems},
  volume={30},
  year={2017}
}

@inproceedings{devlin2019bert,
  title={BERT: Pre-training of deep bidirectional transformers for language understanding},
  author={Devlin, Jacob and Chang, Ming-Wei and Lee, Kenton and Toutanova, Kristina},
  booktitle={Proceedings of the 2019 Conference of the North American Chapter of the Association for Computational Linguistics: Human Language Technologies},
  pages={4171--4186},
  year={2019},
  publisher={Association for Computational Linguistics},
  doi={10.18653/v1/N19-1423},
  url={https://aclanthology.org/N19-1423/}
}

@inproceedings{Shomer_2024, series={KDD ’24},
   title={LPFormer: An Adaptive Graph Transformer for Link Prediction},
   url={http://dx.doi.org/10.1145/3637528.3672025},
   DOI={10.1145/3637528.3672025},
   booktitle={Proceedings of the 30th ACM SIGKDD Conference on Knowledge Discovery and Data Mining},
   publisher={ACM},
   author={Shomer, Harry and Ma, Yao and Mao, Haitao and Li, Juanhui and Wu, Bo and Tang, Jiliang},
   year={2024},
   month=aug, pages={2686–2698},
   collection={KDD ’24} }

@article{brown2020language,
  title={Language models are few-shot learners},
  author={Brown, Tom and Mann, Benjamin and Ryder, Nick and Subbiah, Melanie and Kaplan, Jared D and Dhariwal, Prafulla and Neelakantan, Arvind and Shyam, Pranav and Sastry, Girish and Askell, Amanda and others},
  journal={Advances in Neural Information Processing Systems},
  volume={33},
  pages={1877--1901},
  year={2020}
}

@inproceedings{kipf2017semi,
  title={Semi-supervised classification with graph convolutional networks},
  author={Kipf, Thomas N and Welling, Max},
  booktitle={International Conference on Learning Representations},
  year={2017},
  eprint={1609.02907},
  archivePrefix={arXiv},
  primaryClass={cs.LG},
  url={https://openreview.net/forum?id=SJU4ayYgl}
}

@inproceedings{xu2019gin,
  title={How powerful are graph neural networks?},
  author={Xu, Keyulu and Hu, Weihua and Leskovec, Jure and Jegelka, Stefanie},
  booktitle={International Conference on Learning Representations},
  year={2019}
}

@inproceedings{dosovitskiy2020image,
  title={An image is worth 16x16 words: Transformers for image recognition at scale},
  author={Dosovitskiy, Alexey and Beyer, Lucas and Kolesnikov, Alexander and Weissenborn, Dirk and Zhai, Xiaohua and Unterthiner, Thomas and Dehghani, Mostafa and Minderer, Matthias and Heigold, Georg and Gelly, Sylvain and others},
  booktitle={International Conference on Learning Representations},
  year={2021},
  url={https://openreview.net/forum?id=YicbFdNTTy}
}

@inproceedings{gong2021ast,
  title={AST: Audio spectrogram transformer},
  author={Gong, Yuan and Chung, Yu-An and Glass, James},
  booktitle={Proceedings of Interspeech 2021},
  pages={571--575},
  year={2021},
  doi={10.21437/Interspeech.2021-698}
}

@inproceedings{zhou2021informer,
  title={Informer: Beyond efficient transformer for long sequence time-series forecasting},
  author={Zhou, Haoyi and Zhang, Shanghang and Peng, Jieqi and Zhang, Shuai and Li, Jianbin and Xiong, Hui and Wang, Wancai},
  booktitle={Proceedings of the AAAI Conference on Artificial Intelligence},
  volume={35},
  number={12},
  pages={11106--11115},
  year={2021}
}

@inproceedings{velivckovic2018graph,
  title={Graph attention networks},
  author={Veli{\v{c}}kovi{\'c}, Petar and Cucurull, Guillem and Casanova, Arantxa and Romero, Adriana and Li{\`o}, Pietro and Bengio, Yoshua},
  booktitle={International Conference on Learning Representations},
  year={2018}
}

@inproceedings{ying2021graphormer,
  title={Do Transformers Really Perform Bad for Graph Representation?},
  author={Ying, Chengxuan and Cai, Tianle and Luo, Shengjie and Zheng, Shuxin and Ke, Guolin and He, Di and Shen, Yanming and Liu, Tie-Yan},
  booktitle={Advances in Neural Information Processing Systems},
  volume={34},
  pages={28877--28888},
  year={2021}
}

@inproceedings{liu2024gradformer,
  title={Gradformer: Graph Transformer with Exponential Decay},
  author={Liu, Chuang and Yao, Zelin and Zhan, Yibing and Ma, Xueqi and Pan, Shirui and Hu, Wenbin},
  booktitle={Proceedings of the Thirty-Third International Joint Conference on Artificial Intelligence},
  pages={2171--2179},
  year={2024},
  doi={10.24963/ijcai.2024/240}
}

@article{qiu2022eigenformer,
  title={Graph Transformers without Positional Encodings},
  author={Garg, Ayush},
  journal={arXiv preprint arXiv:2401.17791},
  year={2024},
  eprint={2401.17791},
  archivePrefix={arXiv},
  primaryClass={cs.LG}
}

@inproceedings{zhang2020transgnn,
  title={TransGNN: Harnessing the Collaborative Power of Transformers and Graph Neural Networks for Recommender Systems},
  author={Zhang, Peiyan and Yan, Yuchen and Zhang, Xi and Li, Chaozhuo and Wang, Senzhang and Huang, Feiran and Kim, Sunghun},
  booktitle={Proceedings of the 47th International ACM SIGIR Conference on Research and Development in Information Retrieval},
  pages={1285--1295},
  year={2024}
}

@inproceedings{rampasek2022graphgps,
  title={Recipe for a General, Powerful, Scalable Graph Transformer},
  author={Ramp\'a\v{s}ek, Ladislav and Galkin, Mikhail and Dwivedi, Vijay Prakash and Luu, Anh Tuan and Wolf, Guy and Beaini, Dominique},
  booktitle={Advances in Neural Information Processing Systems},
  volume={35},
  year={2022}
}

@inproceedings{dwivedi2021gt,
  title={A Generalization of Transformer Networks to Graphs},
  author={Dwivedi, Vijay Prakash and Bresson, Xavier},
  booktitle={AAAI Workshop on Deep Learning on Graphs: Methods and Applications},
  year={2021}
}

@inproceedings{beak2021gmt,
  title={Accurate Learning of Graph Representations with Graph Multiset Pooling},
  author={Baek, Jinheon and Kang, Minki and Hwang, Sung Ju},
  booktitle={International Conference on Learning Representations},
  year={2021}
}

@inproceedings{chen2022nagphormer,
  title={NAGphormer: A Tokenized Graph Transformer for Node Classification in Large Graphs},
  author={Chen, Jinsong and Gao, Kaiyuan and Li, Gaichao and He, Kun},
  booktitle={International Conference on Learning Representations},
  year={2023}
}

@inproceedings{chen2022sat,
  title={Structure-Aware Transformer for Graph Representation Learning},
  author={Chen, Dexiong and O'Bray, Leslie and Borgwardt, Karsten},
  booktitle={Proceedings of the 39th International Conference on Machine Learning},
  series={Proceedings of Machine Learning Research},
  volume={162},
  pages={3469--3489},
  year={2022}
}

@inproceedings{deng2024polynormer,
  title={Polynormer: Polynomial-Expressive Graph Transformer in Linear Time},
  author={Deng, Chenhui and Yue, Zichao and Zhang, Zhiru},
  booktitle={International Conference on Learning Representations},
  year={2024}
}

@inproceedings{hamilton2017inductive,
  title={Inductive representation learning on large graphs},
  author={Hamilton, William L and Ying, Zhitao and Leskovec, Jure},
  booktitle={Advances in Neural Information Processing Systems},
  volume={30},
  year={2017}
}

@inproceedings{kreuzer2021san,
  title={Rethinking Graph Transformers with Spectral Attention},
  author={Kreuzer, Devin and Beaini, Dominique and Hamilton, William L. and L\'{e}tourneau, Vincent and Tossou, Prudencio},
  booktitle={Advances in Neural Information Processing Systems},
  volume={34},
  year={2021}
}

@inproceedings{kong2023goat,
  title={GOAT: A Global Transformer on Large-scale Graphs},
  author={Kong, Kezhi and Chen, Jiuhai and Kirchenbauer, John and Ni, Renkun and Bruss, C. Bayan and Goldstein, Tom},
  booktitle={Proceedings of the 40th International Conference on Machine Learning},
  series={Proceedings of Machine Learning Research},
  volume={202},
  pages={17375--17390},
  year={2023},
  organization={PMLR}
}

@article{li2016gated,
  title={Residual Gated Graph ConvNets},
  author={Bresson, Xavier and Laurent, Thomas},
  journal={arXiv preprint arXiv:1711.07553},
  year={2017},
  eprint={1711.07553},
  archivePrefix={arXiv},
  primaryClass={cs.LG}
}

@inproceedings{luo2024classic,
  title={Classic GNNs are Strong Baselines: Reassessing GNNs for Node Classification},
  author={Luo, Yuankai and Shi, Lei and Wu, Xiao-Ming},
  booktitle={Advances in Neural Information Processing Systems},
  year={2024}
}

@inproceedings{liu2024ncn,
  title={Neural Common Neighbor with Completion for Link Prediction},
  author={Wang, Xiyuan and Yang, Haotong and Zhang, Muhan},
  booktitle={International Conference on Learning Representations},
  year={2024},
}

@article{morris2020tudataset,
  title={TUDataset: A collection of benchmark datasets for learning with graphs},
  author={Morris, Christopher and Kriege, Nils M and Bause, Frank and Kersting, Kristian and Mutzel, Petra and Neumann, Marion},
  booktitle={ICML 2020 Workshop on Graph Representation Learning and Beyond (GRL+ 2020)},
  archivePrefix={arXiv},
  eprint={2007.08663},
  url={www.graphlearning.io},
  year={2020}
}

@article{sen2008collective,
  title={Collective classification in network data},
  author={Sen, Prithviraj and Namata, Galileo and Bilgic, Mustafa and Getoor, Lise and Galligher, Brian and Eliassi-Rad, Tina},
  journal={AI Magazine},
  volume={29},
  number={3},
  pages={93--93},
  year={2008}
}

@article{shchur2018pitfalls,
  title={Pitfalls of Graph Neural Network Evaluation},
  author={Shchur, Oleksandr and Mumme, Maximilian and Bojchevski, Aleksandar and G{\"u}nnemann, Stephan},
  journal={arXiv preprint arXiv:1811.05868},
  year={2018},
  eprint={1811.05868},
  archivePrefix={arXiv},
  primaryClass={cs.LG}
}

@article{mernyei2020wikics,
  title={Wiki-CS: A Wikipedia-Based Benchmark for Graph Neural Networks},
  author={Mernyei, P{\'e}ter and Cangea, C{\u{a}}t{\u{a}}lina},
  journal={arXiv preprint arXiv:2007.02901},
  year={2020},
  eprint={2007.02901},
  archivePrefix={arXiv},
  primaryClass={cs.LG}
}

@inproceedings{hu2020ogb,
  title={Open Graph Benchmark: Datasets for Machine Learning on Graphs},
  author={Hu, Weihua and Fey, Matthias and Zitnik, Marinka and Dong, Yuxiao and Ren, Hongyu and Liu, Bowen and Catasta, Michele and Leskovec, Jure},
  booktitle={Advances in Neural Information Processing Systems},
  volume={33},
  year={2020}
}

@inproceedings{shirzad2023exphormer,
  title={Exphormer: Sparse Transformers for Graphs},
  author={Shirzad, Hamed and Velingker, Ameya and Venkatachalam, Balaji and Sutherland, Danica J. and Sinop, Ali Kemal},
  booktitle={Proceedings of the 40th International Conference on Machine Learning},
  series={Proceedings of Machine Learning Research},
  volume={202},
  pages={31613--31632},
  year={2023},
}

@inproceedings{wu2022nodeformer,
  title={NodeFormer: A Scalable Graph Structure Learning Transformer for Node Classification},
  author={Wu, Qitian and Zhao, Wentao and Li, Zenan and Wipf, David and Yan, Junchi},
  booktitle={Advances in Neural Information Processing Systems},
  volume={35},
  year={2022}
}

@inproceedings{wu2021graphtrans,
  title={Graph Transformer for Graph-to-Sequence Learning},
  author={Cai, Deng and Lam, Wai},
  booktitle={Proceedings of the AAAI Conference on Artificial Intelligence},
  volume={34},
  number={05},
  pages={7464--7471},
  year={2020},
  doi={10.1609/aaai.v34i05.6243},
  eprint={1911.07470},
  archivePrefix={arXiv},
  primaryClass={cs.CL}
}

@inproceedings{wu2023sgformer,
  title={SGFormer: Simplifying and Empowering Transformers for Large-Graph Representations},
  author={Wu, Qitian and Zhao, Wentao and Yang, Chenxiao and Zhang, Hengrui and Nie, Fan and Jiang, Haitian and Bian, Yatao and Yan, Junchi},
  booktitle={Advances in Neural Information Processing Systems},
  volume={36},
  year={2023}
}

@inproceedings{yun2021neo,
  title={Neo-GNNs: Neighborhood Overlap-aware Graph Neural Networks for Link Prediction},
  author={Yun, Seongjun and Kim, Seoyoon and Lee, Junhyun and Kang, Jaewoo and Kim, Hyunwoo J.},
  booktitle={Advances in Neural Information Processing Systems},
  volume={34},
  pages={13683--13694},
  year={2021}
}

@inproceedings{zhang2018link,
  title={Link prediction based on graph neural networks},
  author={Zhang, Muhan and Chen, Yixin},
  booktitle={Advances in Neural Information Processing Systems},
  volume={31},
  year={2018}
}

@inproceedings{zhao2023buddy,
  title={Graph Neural Networks for Link Prediction with Subgraph Sketching},
  author={Chamberlain, Benjamin Paul and Shirobokov, Sergey and Rossi, Emanuele and Frasca, Fabrizio and Markovich, Thomas and Hammerla, Nils Yannick and Bronstein, Michael M. and Hansmire, Max},
  booktitle={International Conference on Learning Representations},
  year={2023}
}

@article{Ozmen2024BenchmarkingER,
  author    = {Ozmen, Muberra and Regol, Florence and Markovich, Thomas},
  title     = {Benchmarking Edge Regression on Temporal Networks},
  journal   = {Journal of Data-centric Machine Learning Research},
  year      = {2024},
  url       = {https://openreview.net/forum?id=4k4cocpuSw}
}

\end{document}